\documentclass[runningheads]{llncs}
\usepackage[T1]{fontenc}
\usepackage{xcolor}
\usepackage{amsmath,amssymb}
\usepackage{booktabs}
\usepackage{algorithm}
\usepackage{algorithmicx}
\usepackage{algpseudocode}
\usepackage{verbatim}
\usepackage{fancyvrb}
\usepackage{fvextra}
\usepackage{url}
\usepackage{graphicx}

\usepackage{listings}
\usepackage{caption}
\begin{document}
\title{World2Rules: A Neuro-Symbolic Framework for Learning World-Governing Safety Rules for Aviation}

\titlerunning{World2Rules}
%
\author{Haichuan Wang\inst{1}\and
Jay Patrikar\inst{1} \and
Sebastian Scherer\inst{1}}

\authorrunning{H. Wang et al.}

\institute{Carnegie Mellon University, Pittsburgh PA 15213, USA\\
\email{\{haichuaw,jpatrika,basti\}@andrew.cmu.edu}}

\maketitle              
%


\begin{abstract}
Many real-world safety-critical systems are governed by explicit rules that define unsafe world configurations and constrain agent interactions. In practice, these rules are complex and context-dependent, making manual specification incomplete and error-prone. Learning such rules from real-world multimodal data is further challenged by noise, inconsistency, and sparse failure cases. Neural models can extract structure from text and visual data but lack formal guarantees, while symbolic methods provide verifiability yet are brittle when applied directly to imperfect observations. We present World2Rules, a neuro-symbolic framework for learning world-governing safety rules from real-world multimodal aviation data. World2Rules learns from both nominal operational data and aviation crash and incident reports, treating neural models as proposal mechanisms for candidate symbolic facts and inductive logic programming as a verification layer. The framework employs hierarchical reflective reasoning, enforcing consistency across examples, subsets, and rules to filter unreliable evidence, aggregate only mutually consistent components, and prune unsupported hypotheses. This design limits error propagation from noisy neural extractions and yields compact, interpretable first-order logic rules that characterize unsafe world configurations. We evaluate World2Rules on real-world aviation safety data and show that it learns rules that achieve 23.6\% higher F1 score than purely neural and 43.2\% higher F1 score than single-pass neuro-symbolic baseline, while remaining suitable for safety-critical reasoning and formal analysis.

\end{abstract}


\section{Introduction}
\begin{figure}[t]
    \centering
    \includegraphics[width=\textwidth]{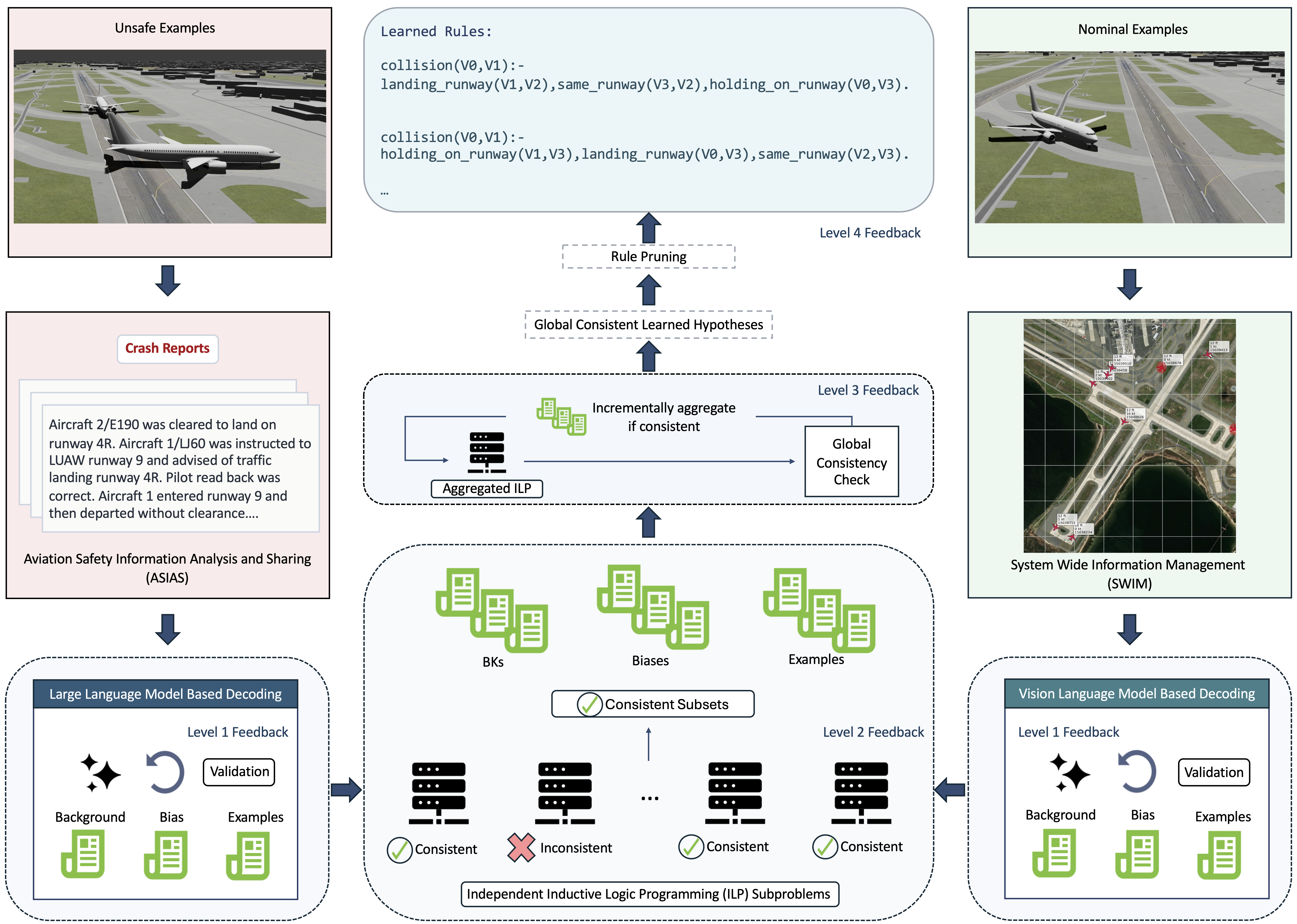}
    \caption{Overview of World2Rules framework. Textual incident reports and visual trajectory data are processed by Large Language Model (LLM) and Vision Language Model (VLM) based extractors to generate Inductive Logic Programming (ILP) inputs, which are decomposed into independent ILP subproblems. Consistency-driven feedback is applied at the example, subset, and rule levels to filter, aggregate, and prune hypotheses, yielding compact and interpretable symbolic safety rules.}
    \label{fig:pipeline}
\end{figure}
Many real-world systems are governed not only by physical processes, but by explicit rules that constrain how agents may interact with their environment. These rules define which configurations of agents, objects, and relations are permissible and which are unsafe. Making such rules explicit is essential for reasoning, verification, and accountability in safety-critical settings. However, in practice, world-governing rules are often implicitly encoded in data, procedures, regulations, or expert knowledge~\cite{patrikar2022predicting}, rather than represented in a form that supports systematic reasoning.

Modern data sources increasingly provide rich, multimodal observations of the world, including text, images, and structured relational data. Together, these signals implicitly describe how a world operates and where its safety boundaries lie. Importantly, such data is often asymmetric: large volumes of nominal, everyday behavior are readily available, while unsafe or failure cases appear sparsely, typically in the form of incident or crash reports. Learning safety-relevant rules from this combination of abundant nominal operational data and sparse, noisy safety violation data poses a fundamental challenge for data-driven methods.

Neural models excel at processing unstructured inputs and extracting salient patterns from both text and visual observations. However, they offer limited guarantees of logical consistency, robustness, or safety. In particular, purely neural approaches struggle to reconcile conflicting evidence across modalities, are vulnerable to spurious correlations, and provide little support for formal verification or counterfactual reasoning, properties critical in aviation. These limitations are amplified when learning from imbalanced data, where rare but safety-critical failure cases must be interpreted in the context of overwhelmingly nominal behavior.

Symbolic rules address these challenges by encoding explicit relational structure and invariants that hold across contexts. By representing knowledge in a logical form, symbolic rules support deductive reasoning, consistency checking, and verifiable inference, while enabling compositional generalization beyond the specific configurations observed during training. Prior work has shown that such representations capture core relational regularities that neural models alone fail to learn reliably, and that they form a foundation for robust, interpretable, and safety-aligned reasoning~\cite{Marra2024,Wang2022,Bouneffouf2022}. Recent methods have demonstrated the effectiveness of injecting hand-specified symbolic rules into learned planners and policies~\cite{patrikar2024rulefuser,aloor2023followtherules}, but these approaches assume the rules are given rather than learned from data.

Recent advances in large foundation models have renewed interest in learning symbolic structures directly from data, treating neural models as generators of candidate symbolic representations~\cite{Padalkar2025}. However, when symbolic rules are induced from noisy, multimodal observations, particularly from rare failure cases, neural proposals alone are insufficient. Hallucinated predicates, inconsistent relations, and unsupported generalizations can easily lead to incorrect or unsafe conclusions~\cite{Ji2023}. In safety-critical settings, such failures are unacceptable.

We present World2Rules, a neuro-symbolic framework for learning world-governing safety rules from multimodal data through reflective reasoning. World2Rules treats pretrained models as proposal mechanisms that extract candidate symbolic facts from multimodal data, while an inductive logic programming (ILP) system acts as a verification layer that reflects feedback onto the learning process~\cite{Cropper2022}. Candidate rules are repeatedly evaluated for logical feasibility, consistency with both nominal operational and safety violation data, and empirical support. Learning is decomposed into independent subproblems that are validated under symbolic constraints, and only mutually consistent components are incrementally aggregated. This reflective, solver-in-the-loop design enforces multi-level consistency, limits error propagation from noisy neural proposals, and yields compact, interpretable rule sets that remain stable under imperfect observations.

A key aspect of World2Rules is its explicit use of both nominal operation and safety violation data. Prior work has shown that crash and incident reports can serve as contrastive evidence for learning safe behaviors~\cite{patrikar2025negativedata}. Rather than treating these sources symmetrically, World2Rules leverages their complementary roles: nominal operational data constrains the space of permissible world configurations, while incident and crash reports provide evidence that delineates unsafe or forbidden states. By enforcing consistency across both types of evidence, the framework learns rules that are not only predictive, but also aligned with the underlying safety structure of the world.

Aviation safety analysis provides a compelling instantiation of this problem, where explicit world-governing rules are essential.
The 2023 runway incursion at JFK occurred when an aircraft failed to comply with the rule of crossing the correct runway, leading to a near miss~\cite{ntsb2024jfkincursion}.
Deviating from normal operation rules or air traffic controller instructions could also cause catastrophic outcomes, such as the accident at Tokyo Haneda, Japan, where two planes collided on the runway~\cite{JTSB2024Haneda}.
Aircraft taxiing, crossing runways, taking off, and landing are coordinated through a tightly constrained set of procedural and regulatory rules designed to prevent unsafe interactions.
Using nominal aircraft trajectory data alongside historical incident and crash reports, we show that treating ILP as a verification and safety layer substantially improves robustness and generalization compared to purely neural or na\"ive neuro-symbolic baselines, while preserving interpretability.

In summary, this work makes the following contributions:
\begin{enumerate}
    \item We introduce World2Rules, a neuro-symbolic framework for learning declarative, world-governing safety rules from multimodal data.
    \item We formalize a hierarchical consistency-checking procedure that integrates pretrained extractors with solver-backed verification.
    \item We demonstrate how combining nominal operational data with safety violation evidence enables robust learning of safety rules for airport surface operations.
\end{enumerate}


\section{Related Work}
\label{sec:relatedWork}

Recent work has explored the integration of large language models (LLMs) with symbolic rule induction, highlighting challenges and opportunities for combining neural perception with formal reasoning.
Yang et al.~\cite{yang2025robusthypothesisgenerationllmautomated} introduce a multi-agent framework where LLMs automatically construct predicate vocabularies and language bias for ILP, reducing reliance on expert-defined symbolic structures.
Peng et al.~\cite{peng2025abductivelogicalruleinduction} combine multimodal LLMs with ILP: their ILP-CoT approach uses LLM-proposed rule structures to prune ILP search spaces and then employs symbolic ILP to generate formally grounded logical rules.

Scaling ILP to larger rule sets remains an open challenge.
Hocquette et al.~\cite{hocquette2024joiner} address this by composing smaller rules via constraint solvers, showing improved predictive performance in structured domains such as game playing.
These methods share the goal of leveraging neural models to support symbolic induction while mitigating noise and search complexity.

Complementary to these efforts, Gandar{\^e}la et al.~\cite{gandarela2025inductivelearninglogicaltheories} present a systematic analysis of inductive learning of logical theories with LLMs, studying how theory expressivity and dependency structure affect performance and highlighting the limitations of directly generating logical programs with LLMs, even when augmented with formal inference feedback, particularly for long relationship chains.

Other neuro-symbolic approaches focus on predicate-level generation and correction. Vision-to-predicate approaches such as Pix2Pred~\cite{athalye2025pixelspredicateslearningsymbolic} generate structured relational predicates from visual inputs, offering a promising way to automatically produce symbolic inputs for ILP. Hu et al.~\cite{Hu_Dai_Jiang_Zhou_2025} propose Abductive Reflection, which leverages abductive reasoning to identify and rectify inconsistencies in neural predictions with respect to domain knowledge. Padalkar et al.~\cite{padalkar2024neurosymbolic} present a neurosymbolic framework for bias correction using class-specific sparse filters to improve interpretability and accuracy. However, in aviation safety analysis, such predicates must be validated and combined under explicit operational rules, motivating solver-backed ILP to enforce global consistency during rule induction.

Domain-specific language models have also been proposed for safety analysis. Aviation-BERT~\cite{doi:10.2514/6.2023-3436} adapts BERT~\cite{devlin2019bertpretrainingdeepbidirectional} to aviation incident reports, improving performance on text classification and information extraction tasks in the aviation domain. However, they do not yield interpretable or formally verifiable rules, limiting their usefulness for rule-violation analysis, which is fundamental to aviation safety.

Unlike these methods, World2Rules derives ILP inputs from raw multimodal data rather than assuming clean symbolic inputs, and enforces consistency at multiple levels of granularity (extraction, subset, aggregation, and rule support) to produce interpretable hypotheses that generalize to safety-critical runway incursion detection.

\section{Problem Setting}
\label{sec:problemSetting}
World2Rules learns declarative safety rules from multimodal data.
We first define the general problem of learning safety constraints as logical rules, then formalize the ILP learning setting.

\subsection{Learning Safety Rules as Logical Constraints}
Consider a world populated by agents, objects, and relations among them.
Some configurations of these entities are safe; others violate world-governing safety constraints.
Our goal is to learn a set of first-order logic rules that define a target predicate
\[ \texttt{target}(X, Y), \]
which holds when entities $X$ and $Y$ participate in a configuration that violates a safety constraint.

Learning these rules requires two complementary data sources: (1) \emph{violation data}, documenting events in which the target predicate holds, and (2) \emph{nominal data}, capturing routine operations in which no violation occurs.
These two sources are independent and unpaired: they need not correspond to the same events, locations, or time periods.
Violation data provides positive evidence that delineates unsafe configurations, while nominal data constrains the space of permissible states and prevents overgeneralization.

We instantiate this problem in the domain of aviation surface operations, where the target predicate becomes
\[ \texttt{collision}(A, B), \]
indicating that aircraft $A$ and $B$ participate in a potential collision or loss-of-separation event.
The violation data consists of aviation incident and crash reports from sources such as the Aviation Safety Information Analysis and Sharing (ASIAS) system~\cite{FAA_ASIAS}, while the nominal data consists of airport surface images overlaid with routine trajectory data from Amelia-48~\cite{navarro2024amelia}.

\subsection{ILP Learning Setting}

In classical ILP, given a hypothesis space $\mathcal{H}$ and constraints $C$, the learner receives positive examples $E^{+}$, negative examples $E^{-}$, and background knowledge $B$ as the tuple $(E^{+}, E^{-}, B)$ and searches for a hypothesis $H \in \mathcal{H}$ that is consistent with the examples.
Unlike classical ILP, our system \emph{derives} $(E^{+}, E^{-}, B)$ from raw multimodal data $D$ using pretrained language and vision-language models.
In our setting, the raw input is a set of data sources:
\[D = \text{(violation data, nominal data)}.\]

Pretrained models extract candidate symbolic facts from $D$.
Violation data yields positive examples $E^{+}$, encoding world states in which the target predicate holds, along with grounded background facts $B$ (e.g., entity identifiers, spatial relations, movement predicates).
Nominal data yields negative examples $E^{-}$, encoding configurations in which the target predicate must not hold.
Because the extracted facts may be incomplete, noisy, or contain hallucinated relations, they cannot be assumed to be mutually consistent.
We therefore treat the extracted outputs as noisy inputs to the ILP pipeline and rely on solver-backed verification (Sec.~\ref{sec:ILPLearning}) to filter inconsistencies and construct a consistent hypothesis.

Domain experts define the set of allowable predicates and their argument types; this vocabulary remains fixed throughout the pipeline.
Given this vocabulary, the pretrained models generate the bias file encoding $\mathcal{H}$.
The constraints $C$ (search bounds, clause length limits, and other solver-level parameters) are also set by domain experts prior to learning.

\begin{definition}[Constructed ILP Instance]
Given raw multimodal data $D$, a domain expert-defined predicate vocabulary, a model-generated bias file encoding hypothesis space $\mathcal{H}$ over that vocabulary, and expert-configured solver constraints $C$, the system constructs an ILP problem instance
\[
(E^{+}, E^{-}, B),
\]
where the positive examples $E^{+}$, negative examples $E^{-}$, and background knowledge $B$
are derived from $D$ via pretrained extractors.
\end{definition}

\section{Our Approach}
\label{sec:ourApproach}

World2Rules combines pretrained extractors with solver-backed inductive logic programming to learn safety rules from multimodal data.
We first describe the domain-specific data extraction pipelines, then the ILP engine, the incremental learning and aggregation procedure, and correctness guarantees.

\subsection{Violation Data Extraction}
\label{sec:PosData}
Positive examples are derived from unstructured aviation safety reports sourced from ASIAS~\cite{FAA_ASIAS}.
As shown in Fig.~\ref{fig:crash_extract}, extraction follows a four-step pipeline shared by both data modalities: (1) preprocessing, (2) entity and relation extraction, (3) context generation, and (4) validation.
\begin{figure}
    \centering
    \includegraphics[width=\linewidth]{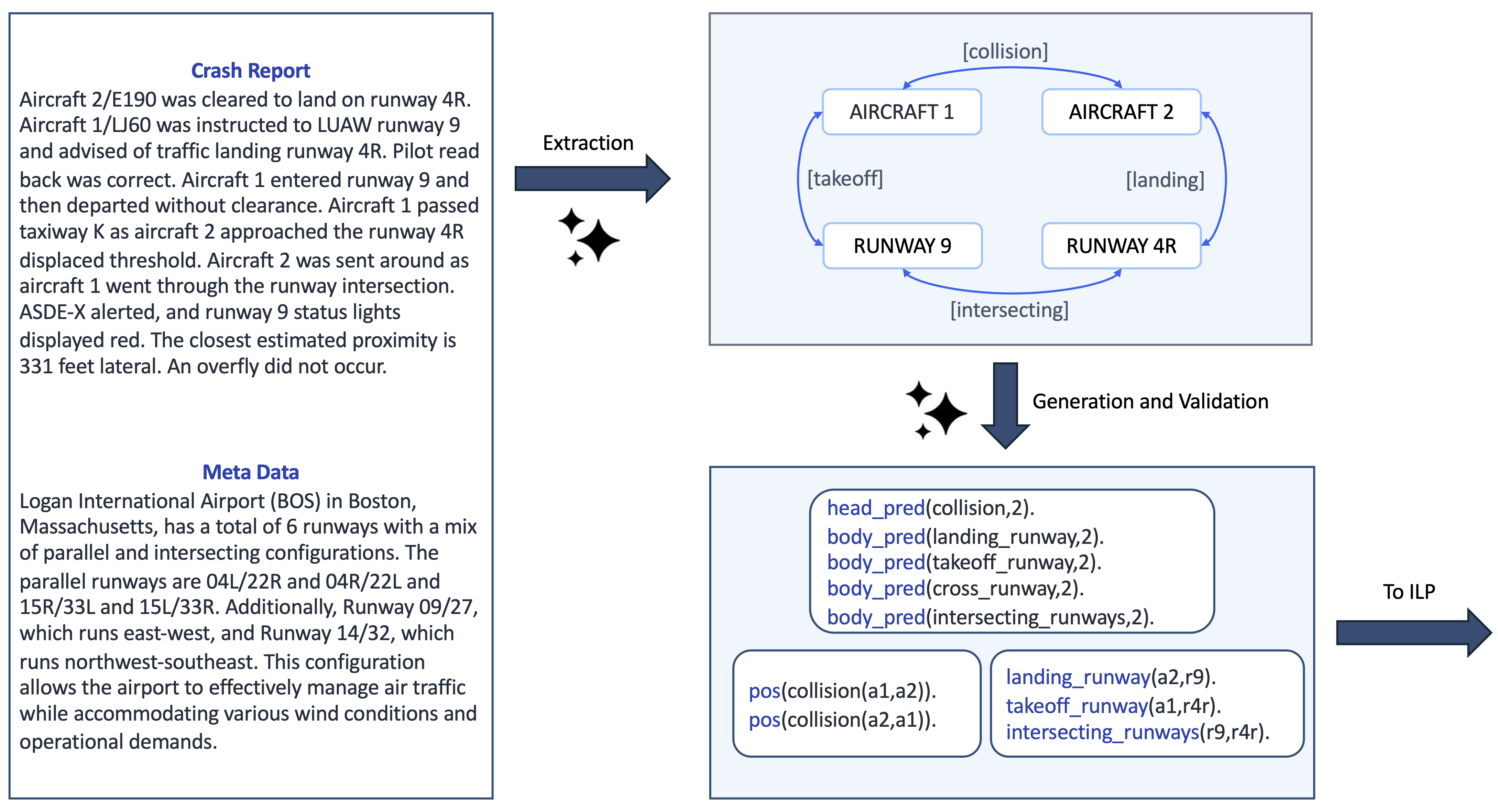}
    \caption{Extraction pipeline for converting crash reports into ILP inputs. Each report is augmented with airport metadata, parsed by an LLM into typed entities and relations, and converted into background knowledge, bias, and example files. Validation ensures syntactic and semantic consistency before downstream learning.}
    \label{fig:crash_extract}
\end{figure}

\paragraph{Preprocessing.}
Each crash report is augmented with airport metadata (runway numbers, geometry, orientation) to improve grounding.

\paragraph{Entity and Relation Extraction.}
An LLM extracts typed entities (e.g., Runway, Vehicle, Aircraft) and their pairwise relations from each report.

\paragraph{Context Generation.}
The LLM then generates the background knowledge, bias, and example files.
The \emph{bias file} encodes the hypothesis space $\mathcal{H}$: it specifies the allowed predicates, their argument types, and the mode declarations that constrain clause construction.
All prompts enforce strict adherence to a fixed master predicate list and include in-context examples to maintain consistency.

\paragraph{Validation.}
\label{sec:LLMValidation}
Each extraction is validated for syntactic and semantic consistency.
If validation fails, the report is reprocessed for up to three attempts; reports that remain invalid are discarded.

\subsection{Nominal Data Extraction}
\label{sec:NegData}
Negative examples are derived from routine airport surface operations using trajectory data from Amelia-48~\cite{navarro2024amelia}.
The extraction follows the same four-step pipeline as violation data, with the following differences.

\paragraph{Preprocessing.}
Trajectory data for each scene is overlaid on the corresponding airport image.
Each runway is annotated with its identifier (e.g., 31L) using hand-labeled references to aid recognition.

\paragraph{Entity and Relation Extraction.}
A VLM replaces the LLM as the extractor, identifying typed entities and their relations directly from the annotated airport images.

\paragraph{Context Generation and Validation.}
\label{sec:VLMValidation}
The VLM generates grounded predicate facts following the same master predicate list and in-context example protocol as the violation data pipeline. Validation and retry logic are identical: up to three attempts, with invalid extractions discarded.

\subsection{Inductive Logic Programming System}
We use Popper~\cite{cropper2021learning} as the underlying ILP engine, configured with RC2~\cite{article} and NuWLS~\cite{Chu_Cai_Luo_2023} solvers.
Given background knowledge $B$ and positive and negative examples $(E^{+},E^{-})$, Popper searches for a set of clauses $H$ that is consistent with the examples under the constraints defined by the mode declarations and type signatures.

World2Rules does not modify Popper itself; our contribution is a novel procedure for constructing ILP inputs from multimodal data and performing incremental rule induction over them.

\subsection{Incremental ILP Learning and Rule Aggregation}
\label{sec:ILPLearning}

Rather than running ILP once over the entire dataset, World2Rules applies a four-level feedback mechanism to ensure robustness against extraction errors, noisy hypotheses, and inconsistent examples.

We define a \emph{subset} as a self-contained ILP instance $(B_i, E_i^{+}, E_i^{-})$ constructed by pairing one violation report (which provides positive examples and associated background knowledge) with one nominal observation (which provides negative examples and associated background knowledge).
Given $M$ such subsets,
\[\{(B_i, E_i^{+}, E_i^{-})\}_{i=1}^{M},\]
Popper is run independently on each subset as a solver-backed consistency check.
Subsets whose background knowledge or examples are inconsistent are discarded, while consistent subsets are aggregated.
A final hypothesis is then learned from the aggregated ILP instance and pruned to form the final program.

\paragraph{Level 1: Extraction Consistency Feedback.}
Before any ILP learning occurs, each extraction undergoes
consistency checking (Sections~\ref{sec:LLMValidation}, \ref{sec:VLMValidation}). Only syntactically and semantically valid Prolog bundles proceed to ILP learning.

\paragraph{Level 2: Subset-Level ILP Feedback.}
Each subset $(B_i, E_i^{+}, E_i^{-})$ is treated as an independent ILP task. The resulting $H_i$ is run against that subset to ensure consistency and completeness, i.e.
\[
\forall e \in E_i^{+},\; B_i \cup H_i \models e,
\qquad
\forall e \in E_i^{-},\; B_i \cup H_i \not\models e.
\]

If Popper fails to find a hypothesis, the subset is discarded. This step ensures that downstream aggregation is not polluted by subsets whose data is fundamentally inconsistent.

\paragraph{Level 3: Iterative Aggregation of Globally Consistent Subsets.}
After filtering for subset-level consistency, the system performs an iterative aggregation procedure. Let $\{(B_i, E_i^{+}, E_i^{-})\}_{i=1}^{M}$ be the reliable subsets from Level~2. The goal is to construct an aggregated ILP problem whose hypothesis achieves $100\%$ accuracy on all incorporated data.

We maintain a set $S$ of accepted subset indices and aggregated components
\[B_S, \qquad E_S^{+}, \qquad E_S^{-},\] initially all empty. At each iteration, we consider a candidate subset $j\notin S$ and form temporary unions:
\[
B' = B_S \cup B_j, \qquad
E^{+}{}' = E_S^{+} \cup E_j^{+}, \qquad
E^{-}{}' = E_S^{-} \cup E_j^{-}.
\]
We then run Popper on the aggregated ILP instance $(B', E^{+}{}', E^{-}{}')$, obtaining a hypothesis $H'$. The candidate subset $j$ is accepted if and only if Popper finds a non-empty hypothesis that achieves 100\% accuracy on the aggregated examples, with Popper recomputing the hypothesis from scratch at each step:
\[
B' \cup H' \models E^{+}{}'
\qquad \text{and} \qquad
B' \cup H' \not\models E^{-}{}'.
\]
If this condition holds, we update
\[
S \gets S \cup \{j\},\qquad
B_S \gets B',\qquad
E_S^{+} \gets E^{+}{}',\qquad
E_S^{-} \gets E^{-}{}',\qquad
H_S \gets H',
\]

If aggregation fails, we progressively remove examples from the candidate subset until consistency is restored. If consistency is achieved after removing some examples, the remaining partial subset is retained; otherwise, the entire subset $j$ is discarded. This process is repeated until all subsets have been evaluated.

This iterative global-consistency loop assumes that the already-accepted subsets are more reliable than any newly tested subset: a subset is rejected precisely when its inclusion would reduce global accuracy below $100\%$. In effect, the system constructs the largest mutually consistent collection of extracted ILP components and learns a hypothesis from their aggregated ILP problem.

Because such aggregation can be sensitive to the order in which subsets are considered, we introduce a shuffle-based fallback mechanism. Aggregation is first attempted using a default chronological order of the violation reports, reflecting the natural accumulation of safety knowledge over time. If the resulting aggregation discards more than a fraction $\rho$ of reliable subsets, the procedure is repeated using randomly shuffled orderings. This retry process is performed for a bounded number of attempts.

\paragraph{Level 4: Rule-Level Support Pruning.}
Following the aggregation step, each rule $r \in H_S$ is evaluated using the empirical support metric. For a rule $r$, we define its positive support as:
\[
\operatorname{supp}(r)
    = \left| \{ e \in E^{+} \mid B \cup \{r\} \models e \} \right|.
\]

Rules with very low support tend to correspond to overspecialized or
spurious patterns arising from noise in the data sources. Pruning is applied once globally to the aggregated rule set:
\[
\operatorname{supp}(r)
    \;\ge\; \tau \cdot \max_{r' \in H_S} \operatorname{supp}(r'),
    \quad\text{with }\tau = 0.20.
\]

Rules failing to meet this threshold are removed. The remaining high-support rules constitute the final hypothesis used for evaluation on external datasets.
Note that $\rho$ (the subset discard fraction threshold at Level~3) and $\tau$ (the support pruning threshold at Level~4) are distinct parameters controlling different stages of the pipeline.

\begin{algorithm}[htbp]
\caption{Subset-Level Consistency Checking}
\label{alg:subset-check}
\footnotesize
\begin{algorithmic}[1]
\State $\mathcal{R}$ \Comment{set of reliable subset indices}
\State $\mathcal{R} \gets \emptyset$
\For{$i = 1$ to $M$}
    \State $(B_i, E_i^{+}, E_i^{-}) \gets \textsc{Decode}(D_i)$
    \State $H_i \gets \textsc{Popper}(B_i, E_i^{+}, E_i^{-})$
    \If{$H_i = \emptyset$
        \textbf{or} $B_i \cup H_i \not\models E_i^{+}$
        \textbf{or} $B_i \cup H_i \models E_i^{-}$}
        \State discard subset $i$ as unreliable
    \Else
        \State $\mathcal{R} \gets \mathcal{R} \cup \{i\}$
    \EndIf
\EndFor
\State \Return $\mathcal{R}$
\end{algorithmic}
\end{algorithm}

\begin{algorithm}[htbp]
\caption{Iterative Global Aggregation with Support Pruning}
\label{alg:global-agg}
\footnotesize
\begin{algorithmic}[1]
\Require Reliable subset indices $\mathcal{R} \subseteq \{1,\dots,M\}$, subsets $\{(B_i, E_i^{+}, E_i^{-})\}_{i \in \mathcal{R}}$
\Require Maximum retries $T$, failure threshold $\rho$
\Ensure Final pruned hypothesis $H_{\mathrm{final}}$

\State $S^{*} \gets \emptyset$, \quad $H^{*} \gets \emptyset$
\State $B^{*} \gets \emptyset$, \quad $E^{+*} \gets \emptyset$, \quad $E^{-*} \gets \emptyset$
\State $k^{*} \gets -1$, \quad $\textit{success}^{*} \gets \textbf{false}$

\For{$t = 1$ to $T$}
    \If{$t = 1$}
        \State order $\mathcal{R}$ deterministically based on chronological order
    \Else
        \State randomly shuffle $\mathcal{R}$
    \EndIf

    \State $S \gets \emptyset$
    \State $B_S \gets \emptyset$, \quad $E_S^{+} \gets \emptyset$, \quad $E_S^{-} \gets \emptyset$
    \State $H_S \gets \emptyset$

    \For{each $j \in \mathcal{R}$ in the current order}
        \State $B' \gets B_S \cup B_j$
        \State $E^{+}{}' \gets E_S^{+} \cup E_j^{+}$
        \State $E^{-}{}' \gets E_S^{-} \cup E_j^{-}$
        \State $H' \gets \textsc{Popper}(B', E^{+}{}', E^{-}{}')$
        \If{$H' \neq \emptyset$
            \textbf{and} $B' \cup H' \models E^{+}{}'$
            \textbf{and} $B' \cup H' \not\models E^{-}{}'$}
            \State $S \gets S \cup \{j\}$
            \State $B_S \gets B'$, \quad $E_S^{+} \gets E^{+}{}'$, \quad $E_S^{-} \gets E^{-}{}'$
            \State $H_S \gets H'$
        \EndIf
    \EndFor

    \State $k \gets |S|$, \quad $\textit{fail\_frac} \gets 1 - \frac{k}{|\mathcal{R}|}$
    \State $\textit{success} \gets (H_S \neq \emptyset)$

    \If{$(\textit{success} \land \neg \textit{success}^{*})$
        \textbf{or} $((\textit{success} = \textit{success}^{*}) \land (k > k^{*}))$}
        \State $S^{*} \gets S$, \quad $H^{*} \gets H_S$, \quad $k^{*} \gets k$, \quad $\textit{success}^{*} \gets \textit{success}$
        \State $B^{*} \gets B_S$, \quad $E^{+*} \gets E_S^{+}$, \quad $E^{-*} \gets E_S^{-}$
    \EndIf

    \If{$\textit{fail\_frac} \le \rho$}
        \State \textbf{break}
    \EndIf
\EndFor

\State $H_{\mathrm{final}} \gets \textsc{PruneBySupport}(H^{*}, B^{*}, E^{+*})$
\State \Return $H_{\mathrm{final}}$
\end{algorithmic}
\end{algorithm}

\par\medskip
\noindent
This four-level feedback architecture combines pretrained extraction with symbolic induction, allowing the system to correct extraction mistakes, filter inconsistent logic programs, and retain only rules that exhibit stable predictive behavior across the dataset.
Algorithms~\ref{alg:subset-check} and~\ref{alg:global-agg} summarize this pipeline.

\subsection{Correctness}
We prove the correctness of the World2Rules prior to rule pruning. Following rule-level support pruning, the hypothesis remains negative-safe, but no longer carries a completeness guarantee with respect to positive examples.

A rule (or hypothesis) is \emph{negative-safe} if, combined with the background knowledge, it does not entail any negative example.

\begin{theorem}[Pre-Pruning Training Correctness]

Let $(E^{+}, E^{-})$ denote the aggregated positive and negative examples accepted by the four-level feedback mechanism, and let $H_{\mathrm{agg}}$ be the hypothesis produced at the end of Level~3 (before pruning). If every rule admitted during aggregation is negative-safe, then $H_{\mathrm{agg}}$ is training-correct on $(E^{+}, E^{-})$.

\end{theorem}

\begin{proof}
We argue correctness by establishing soundness with respect to negative examples and completeness with respect to positive coverage at the end of Level~3.

\emph{Soundness.}
At Levels~1 and~2, any subset whose background, bias, examples, or induced hypothesis violates consistency or is syntactically invalid is discarded. In particular, Popper is only allowed to return hypotheses $H_i$ satisfying
\[B_i \cup H_i \models E_i^{+}\quad\text{and}\quad
B_i \cup H_i \not\models E_i^{-}.\]
Thus, every retained subset is individually negative-safe. At Level~3, a candidate subset is incorporated into the aggregated ILP instance only if Popper finds a hypothesis that remains negative-safe on the union of all accepted negative examples. Therefore, no aggregation step can introduce a rule that covers a negative example, and soundness with respect to $E^{-}$ is preserved throughout execution.

\emph{Completeness w.r.t. aggregated positives.}
During aggregation, a subset is accepted only when the aggregated hypothesis achieves $100\%$ accuracy on the union of all accepted positive examples. Consequently, the final aggregated hypothesis $H_{\mathrm{agg}}$ satisfies
\[B \cup H_{\mathrm{agg}} \models E^{+}.\]

Combining soundness and completeness, we conclude that $H_{\mathrm{agg}}$ is training-correct on $(E^{+}, E^{-})$.

\end{proof}

We next show that the subsequent pruning step preserves safety, even though it may reduce completeness.

\begin{proposition}[Safety Preservation under Support Pruning]

Let $H_{\mathrm{final}}$ be the hypothesis obtained by applying support-based pruning to $H_{\mathrm{agg}}$. Then $H_{\mathrm{final}}$ remains sound with respect to negative examples, i.e.,
\[
B \cup H_{\mathrm{final}} \not\models E^{-}.
\]
\end{proposition}

\begin{proof}
Support-based pruning removes rules from $H_{\mathrm{agg}}$ but does not introduce new rules or relax negative-safety constraints. Since soundness holds for $H_{\mathrm{agg}}$, removing rules cannot cause any negative example to become provable. Therefore, $H_{\mathrm{final}}$ remains negative-safe.
\end{proof}

\section{Experimental Results}
\label{sec:result}

\subsection{Evaluation Dataset and Protocol}
\label{subsec:eval_protocol}

We curated a held-out test set consisting of 38 independent runway-incursion scenarios. The selected scenarios include a diverse set of real world runway incursion events drawn from operational reports, including incidents such as the 2023 runway incursion at JFK.

\paragraph{Test set construction.}
The evaluation set combines two sources.
First, we draw real-world incidents from publicly documented Category~A and~B runway incursions~\cite{faaRunwayIncursions}.
The availability of such cases with sufficient detail for symbolic annotation is limited, so we supplement them with manually constructed canonical scenarios designed to capture common and safety-critical incursion patterns. Of the 38 scenarios, 28 are real-world cases drawn from a geographically diverse set of airports with both parallel and intersecting runway configurations. The remaining 10 are canonical scenarios covering patterns such as failure to vacate an active runway prior to the arrival of succeeding traffic and unauthorized crossing onto an active runway during an active takeoff or landing roll.

\paragraph{Annotation and verification.}
For each scenario, we manually annotated background knowledge, bias, and positive and negative examples based on real operational contexts, without relying on LLM or VLM extraction.
The resulting test set contains 44 positive incursions and 44 negative examples.
Although the real-world data distribution is inherently imbalanced, this balanced evaluation design prevents metric inflation due to majority-class bias and ensures that false positives and false negatives are penalized equally.

To verify the logical realizability of these annotations, we constructed a hand-engineered rule set using the fixed predicate vocabulary.
This rule set achieves perfect accuracy on the held-out test set, confirming that all scenarios are representable under the assumed symbolic representation.
This verification validates the annotation framework rather than serving as a comparative baseline.

\paragraph{Metrics.}
Each learned hypothesis is evaluated by logical entailment: a scenario is correct if the hypothesis entails all positive examples and rejects all negative examples.
We report accuracy, precision, recall, and F1 score over the full test set.

\subsection{System Variants}
\label{subsec:variants}

We compare three system variants reflecting increasing levels of symbolic reasoning and consistency enforcement:

\begin{enumerate}
    \item \textbf{LLM Only (GPT-4o~\cite{openai2024gpt4o})}:
    The LLM directly generates logical rules from violation report text, without symbolic verification or nominal data.

    \item \textbf{Na\"ive ILP (Pruning Only)}:
    Positive and negative examples are extracted as in Sections~\ref{sec:PosData} and~\ref{sec:NegData}. ILP is run once on the full extracted components with pruning, but without subset-level filtering or iterative consistency enforcement.

    \item \textbf{World2Rules (Ours)}:
    The full pipeline described in Section~\ref{sec:ourApproach}, including extraction validation, subset-level validation, iterative aggregation with global consistency enforcement, and rule-support pruning.
\end{enumerate}

All ILP variants use identical background predicates, mode declarations, and hypothesis constraints.
Performance differences therefore reflect the impact of solver-enforced feedback rather than changes in expressivity.

\subsection{Results}
\label{subsec:results}

\paragraph{Method comparison.}
We evaluate all three system variants trained on nominal data and 300 violation reports (Fig.~\ref{fig:exp1_comparison}).
World2Rules achieves an F1 score of 94.0\%, outperforming the LLM-only baseline (70.4\%) by 23.6 percentage points and the na\"ive ILP baseline (50.8\%) by 43.2 percentage points.
Both ILP variants achieve perfect or near-perfect precision, indicating that solver-backed induction avoids false positives; the key differentiator is recall, where iterative consistency enforcement enables World2Rules to capture a substantially broader range of incursion patterns.

\begin{figure}[t]
    \centering
    \includegraphics[width=\textwidth]{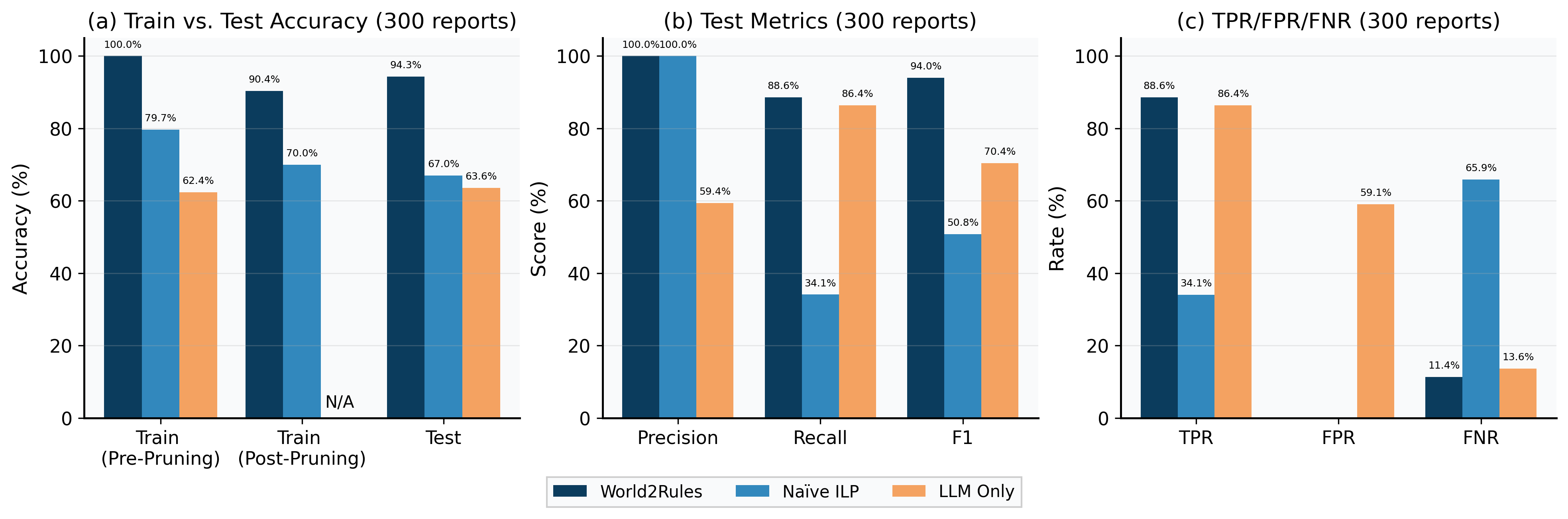}
    \caption{
    Performance comparison across system variants using 300 violation reports. World2Rules achieves 94.0\% F1, outperforming the LLM-only baseline (70.4\%) and na\"ive ILP (50.8\%) while maintaining perfect precision.
    }
    \label{fig:exp1_comparison}
\end{figure}

\paragraph{Data scaling.}
To assess how World2Rules scales with training data, we subsample the violation report corpus at four sizes: 10, 100, 200, and 300 reports (Fig.~\ref{fig:exp2_ablation}). Nominal observations are subsampled and paired with each report to form subsets, as described in Sec.~\ref{sec:ILPLearning}. F1 score improves monotonically from 52.5\% at 10 reports to 94.0\% at 300, driven primarily by gains in recall (36.4\% to 88.6\%) while precision remains at or near 100\%.

\begin{figure}[t]
    \centering
    \includegraphics[width=\textwidth]{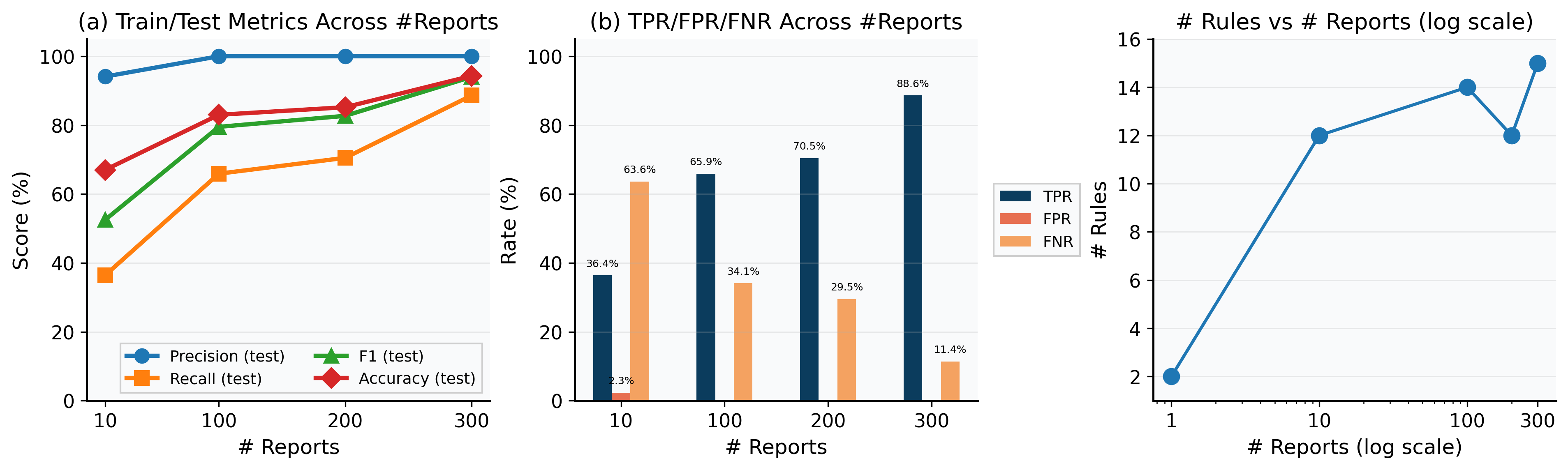}
    \caption{
    Data scaling for World2Rules. F1 score improves from 52.5\% (10 reports) to 94.0\% (300 reports), driven by recall gains while precision remains near 100\%.
    }
    \label{fig:exp2_ablation}
\end{figure}

\subsection{Qualitative Analysis}
\label{subsec:qualitative}

Figure~\ref{fig:exp3_rules} shows representative collision rules learned by World2Rules from 300 violation reports.
Each rule captures a distinct incursion pattern:
Rule~1 flags a collision when one aircraft is landing while another holds on the same runway.
Rule~2 captures an aircraft crossing a runway on which another is landing.
Rule~3 identifies two aircraft on extended areas of the same runway, one landing and one already present. 
The final hypothesis consists of 15 rules.
These rules are compact, interpretable, and correspond to recognized runway incursion categories.

\lstdefinestyle{rules}{ basicstyle=\ttfamily\scriptsize, columns=fullflexible, breaklines=true, breakatwhitespace=false, showstringspaces=false, frame=single, xleftmargin=0.4em, xrightmargin=0.4em, aboveskip=0.4em, belowskip=0.2em, }

\begin{figure}[htbp]
\centering
\captionsetup{font=small}
\begin{lstlisting}[style=rules]
collision(V0,V1):- landing_runway(V1,V2),same_runway(V3,V2),holding_on_runway(V0,V3).
collision(V0,V1):- landing_runway(V1,V2),cross_runway(V0,V2).
collision(V0,V1):- on_extended_area_runway(V1,V2),landing_runway(V0,V3),same_runway(V2,V3).
\end{lstlisting}
\caption{Sample collision rules learned by World2Rules. Each rule captures a distinct runway incursion pattern: holding on an active runway (Rule~1), crossing during landing (Rule~2), and co-occupation of a runway's extended area (Rule~3).}
\label{fig:exp3_rules}
\end{figure}

\subsection{Discussion}
\label{subsec:discussion}

Solver-backed verification plays a critical role in inducing reliable symbolic rules for runway incursion detection.
Across all experiments, World2Rules achieves high precision while substantially improving F1 relative to both purely neural and single-pass ILP baselines.
This balance is particularly important in safety-critical contexts, where systems must avoid spurious hazard detections while reliably identifying genuinely unsafe configurations.

The comparison between na\"ive ILP and World2Rules highlights a key limitation of single-pass neuro-symbolic integration.
Although na\"ive ILP achieves zero false positives, its low recall indicates that conservative hypotheses induced from aggregated, unverified extracted data fail to generalize to diverse operational scenarios.
Iterative consistency enforcement enables the solver to reject unsupported symbolic evidence early and accumulate only logically coherent training subsets.

The data scaling experiment shows that increased training diversity directly benefits rule induction under solver-backed verification.
The proposed aggregation strategy effectively leverages larger corpora without introducing spurious rules.
The improvement in recall with increasing data suggests that the induced hypotheses capture a broader set of incursion patterns rather than overfitting to a narrow subset of scenarios.
The size of the induced rule set stabilizes as training diversity increases, consistent with aggregation refining existing constraints rather than introducing redundant hypotheses.

The observed non-monotonicity in rule count arises solely during the final pruning stage.
Prior to pruning, the raw hypotheses grow monotonically with additional violation reports.
Pruning removes redundant or weakly supported constraints as additional evidence refines the learned rule set.

\subsection{Limitations and Future Work}

\begin{enumerate}
    \item \textbf{Evaluation scale.} The test set is constrained by the limited availability of publicly documented major runway incursions with sufficient detail for symbolic annotation. Future work should expand evaluation as additional incident reports become available.
    \item \textbf{Static relational model.} The current framework models static relational configurations and does not explicitly capture temporal evolution. Incorporating temporal reasoning and uncertainty-aware predicates is an important direction.
    \item \textbf{Fixed predicate vocabulary.} World2Rules relies on expert-defined predicates, which may limit expressiveness. The current evaluation uses 13 predicates; scaling to significantly larger vocabularies or more complex relational structures would increase the ILP search space and may require additional search heuristics or vocabulary decomposition strategies. Extending the framework to support learned or dynamically expanded predicates would improve adaptability.
\end{enumerate}

Beyond aviation, World2Rules can be applied to other domains governed by structured safety rules.
The framework relies on a predicate vocabulary and constraint schema that define the hypothesis space; in new domains, this vocabulary can be redefined to capture relevant entities, relations, and safety conditions, and the same solver-backed learning procedure applies without modification.


\section{Conclusion}
\label{sec:conclusion}
World2Rules learns interpretable first-order logic safety rules from multimodal data by combining pretrained extractors with solver-backed inductive logic programming.
By treating pretrained models as proposal mechanisms and ILP as a verification layer, the framework enforces consistency at multiple levels of granularity, from individual extractions through global rule aggregation.
On real-world aviation safety data, World2Rules achieves 94.0\% F1, outperforming a purely neural baseline by 23.6 and a single-pass neuro-symbolic baseline by 43.2 percentage points, while producing compact, human-readable rules that correspond to recognized runway incursion patterns.
Although we instantiate the framework for aviation surface operations, the approach is domain-agnostic: given a predicate vocabulary and multimodal data sources, the same solver-backed learning procedure can synthesize verifiable safety rules in any domain governed by structured constraints.


\begin{credits}
\subsubsection{\ackname} This work was supported by Boeing (award \#2022UIPA422). This work used Bridges-2 at Pittsburgh Supercomputing Center through allocation \verb|cis220039p| from the Advanced Cyberinfrastructure Coordination Ecosystem: Services \& Support (ACCESS) program, which is supported by National Science Foundation grants \#2138259, \#2138286, \#2138307, \#2137603, and \#2138296. The authors thank Bowen Li for helpful discussions and feedback that contributed to this work. The authors also acknowledge the use of generative AI tools to assist with editing portions of the manuscript and with software development. All generated content was reviewed, verified, and edited by the authors.

\end{credits}


\bibliographystyle{splncs04}
\bibliography{ref}  

\section{Appendix}
\appendix

\subsection{A. Predicate Vocabulary}

\begin{center}
\begin{Verbatim}[frame=single,fontsize=\small]
collision(Agent1, Agent2).
pos(collision(Agent1, Agent2)).
neg(collision(Agent1, Agent2)).
landing_runway(Agent1, Runway1).
takeoff_runway(Agent1, Runway1).
cross_runway(Agent1, Runway1).
on_taxiway(Agent1).
holding_short_runway(Agent1, Runway1).
on_extended_area_runway(Agent1, Runway1).
holding_on_runway(Agent1, Runway1).
parallel_runways(Runway1, Runway2).
intersecting_runways(Runway1, Runway2).
same_runway(Runway1, Runway2).
\end{Verbatim}
\captionof{figure}{Predicate vocabulary used for all Prolog extraction tasks.}
\end{center}

\end{document}